\documentclass[twoside,11pt]{article}

\usepackage{jmlrwcp2e}
\usepackage{url}
\usepackage{amsmath,dsfont}

\newcommand{\bm}{\mathds{1}}
\newcommand{\sgn}{\mathrm{sgn}}

\jmlrheading{1}{2014}{1-16}{}{}{Wojciech Kot{\l}owski}{HEPML workshop at NIPS 2014}

\ShortHeadings{Consistent optimization of AMS by logistic loss minimization}{Wojciech Kot{\l}owski}
\firstpageno{1}

\newcommand{\ams}{\mathrm{AMS}}
\DeclareMathOperator*{\argmin}{arg\,min}
\DeclareMathOperator*{\argmax}{arg\,max}

\begin{document}

\title{Consistent optimization of AMS by logistic loss minimization}

\author{\name Wojciech Kot{\l}owski \email wkotlowski@cs.put.poznan.pl \\
       \addr Pozna{\'n} University of Technology, Poland}

\editor{}

\maketitle

\begin{abstract}
In this paper, we theoretically justify an approach popular
among participants of the Higgs Boson Machine Learning Challenge to optimize
approximate median significance (AMS). 
The approach is based on the following two-stage procedure. First, a real-valued function
$f$ is learned by minimizing a surrogate loss for binary classification, such
as logistic loss, on the training sample. Then, given $f$, a threshold $\hat{\theta}$ is
tuned on a separate validation sample, by direct optimization of AMS.
We show that the regret of the resulting classifier (obtained from thresholding $f$ on $\hat{\theta}$) measured with respect
to the squared AMS, is upperbounded by the regret of $f$ measured with respect
to the logistic loss. Hence, we prove that minimizing logistic surrogate is
a consistent method of optimizing AMS. 
\end{abstract}


\section{Introduction}
\label{sec:introduction}

This paper concerns a problem of learning a classifier to optimize 
approximate median significance (AMS), which was the goal of
the Higgs Boson Machine Learning Challenge (HiggsML), hosted by Kaggle
website (see \cite{AdamBourdarios_etal2014} for details on this contest
and description of the problem).

In particular, we are interested in an approach to optimize AMS, based 
on the following two-stage procedure. First, a real-valued function
$f$ is learned by minimizing a surrogate loss for binary classification, such
as logistic loss function, on the training sample. In the second stage, given $f$, a threshold is
tuned on a separate ``validation'' sample, by direct optimization of AMS with respect to
a classifier obtained from $f$ by classifying all observations with value of $f$ above the threshold
as positive class (signal event), and all observations below the threshold as negative class (background event).

This approach became very popular among HiggsML challenge participants, mainly due to
the fact that its first stage, learning a classifier, does not exploit the task evaluation
metric (AMS) in any way and thus can employ without modifications any standard classification tools such as
logistic regression, LogitBoost, Stochastic Gradient Boosting, Random Forest, etc.
(see, e.g., \cite{FriedmanHastieTibshirani03}). Despite its simplicity, 
this approach proved to be very effective in achieving high leaderboard score in HiggsML.
\footnote{See the HiggsML forum \url{https://www.kaggle.com/c/higgs-boson/forums} for discussions and
presentation of the top score solutions.}

The intuition behind this approach is clear: minimization of logistic loss results in estimation of
conditional probabilities of signal and background event, 
and the AMS is assumed to be maximized by classifying
the events most likely to be signal as signal events.

This paper formalizes this intuition by showing that the approach described above constitutes
a consistent method of optimizing AMS. More specifically, 
we use the notion of \emph{regret} with respect
to some evaluation metric, which is a difference between the performance of a given classifier
and the performance of the optimal classifier with respect to this metric.
Given a function $f$, and a classifier $h_{f,\hat{\theta}}$ obtained from $f$ by thresholding $f$ at
$\hat{\theta}$, we give a bound on the regret of $h_{f,\hat{\theta}}$ measured  with respect to the squared AMS
by the regret of $f$ measured with respect to the
logistic loss, given that the threshold $\hat{\theta}$ is tuned by optimization of AMS among
all classifiers of the form $h_{f,\theta}$ for any threshold value $\theta$.

To our knowledge, this is the first regret bound of this form applicable to a non-decomposable
performance measure such as AMS. We also discuss generalization of our approach to different
performance measures and surrogate loss functions.

\paragraph{Related work.} The issue of consistent optimization of performance
measures which are functions of true positive and true negative rates
has received increasing attention recently in machine learning community
\citep{Narasimhan_etal2014,Natarajan_etal2014,Zhao_etal_2013}.
However, these works are mainly concerned with \emph{statistical consistency} also
known as \emph{calibration}, which determines whether convergence to the minimizer of a 
surrogate loss implies convergence to the minimizer of the task performance measure as sample size
goes to infinity. Here we give a much stronger result which bounds the regret with respect
to squared AMS by the regret with respect to logistic loss. Our result is valid for all
finite sample sizes and informs about the rates of convergence.

Recently, \cite{MackeyBryan2014} proposed a classification cascade approach to optimize
AMS. Their method, based on the theory of Fenchel's duality, iteratively
alternates between solving a cost-sensitive binary classification problem and 
updating misclassification costs. In contrast, the method described here requires
solving an ordinary binary classification problem just once.

\paragraph{Outline.} The paper is organized as follows. In Section \ref{sec:problem_setting},
we introduce basic concepts needed to state our main result presented 
in Section \ref{sec:main_result} and proved in Section \ref{sec:proof}.
Section \ref{sec:generalization}
discusses generalization of our results beyond AMS and logistic loss.

\section{Problem Setting}
\label{sec:problem_setting}

\paragraph{Binary classifier.}
In binary classification, the goal is, given an input (feature vector) $x \in X$,
to accurately predict the output (label) $y \in \{-1,1\}$. We assume input-output
pairs $(x,y)$, which we call \emph{observations}, are generated i.i.d. according to $\Pr(x,y)$.\footnote{
The original HiggsML problem also involved observations' weights, but without loss
of generality, they can be incorporated into the distribution $\Pr(x,y)$.}
A \emph{classifier} is a mapping  $h \colon X \to \{-1,1\}$. Given $h$, we define
the following two quantities:
\[
s(h) = \Pr(h(x) = 1, y=1),\qquad b(h) = \Pr(h(x) = 1,y=-1),
\]
which can be interpreted as true positive and false positive rates of $h$.

\paragraph{AMS and regret.}
Given a classifier $h$, define its \emph{approximate median significance} (AMS) score
\citep{Cowan_etal2011} as
$\ams(h) = \ams(s(h),b(h))$, where:\footnote{Comparing to the definition in \citep{AdamBourdarios_etal2014},
we skip the regularization term $b_{\mathrm{reg}}$. This comes without loss of generality, as
$b_{\mathrm{reg}}$ can be incorporated into $b$ and, since it affects all classifiers equally,
will vanish in the definition of regret.}
\[
\ams(s,b) = \sqrt{2\left((s+b)\log\left(1 + \frac{s}{b}\right) - s\right)}.
\]
It is easier to deal with a \emph{squared} AMS, $\ams^2(h)$, and this quantity
is used throughout the paper.
It is easy to verify that $\ams^2(s,b)$ is increasing in $s$ and decreasing in $b$.
Moreover, $\ams^2(s,b)$ is jointly convex with respect to $(s,b)$.

Let $h^*_{\ams}$ be the classifier which maximizes the $\ams^2$ over all possible
classifiers:
\[
 h^*_{\ams} = \argmax_{h \in \{-1,1\}^X} \ams^2(h).
\]
Given $h$, we define its \emph{AMS regret} as the distance of $h$ from the optimal
classifier $h^*_{\ams}$ measured by means of $\ams^2$:
\[
R_{\ams}(h) = \ams^2(h^*_{\ams}) - \ams^2(h). 
\]

\paragraph{Logistic loss and logistic regret.}

Given a real number $f$, and a label $y$, we define the logistic loss
$\ell_{\log} \colon \{-1,1\} \times \mathbb{R} \to \mathbb{R}_+$ as:
\[
 \ell_{\log}(y,f) = \log\left(1+e^{-yf}\right).
\]
The logistic loss is a commonly used surrogate loss function for binary
classification, employed in various learning methods, such as logistic regression, LogitBoost
or Stochastic Gradient Boosting (see, e.g., \cite{FriedmanHastieTibshirani03}). 
It is convex in $f$, so minimizing logistic
loss over the training sample becomes a convex optimization problem, which can be
solved efficiently. Another advantage of logistic loss is that the sigmoid transform of $f$,
$(1+e^{-f})^{-1}$, can be
used to obtain probability estimates $\Pr(y|x)$.

Given a real-valued function $f \colon X \to \mathbb{R}$, its expected logistic
loss $L_{\log}(f)$ is defined as:
\[
L_{\log}(f) = \mathbb{E}_{(x,y)}[\ell_{\log}(y,f(x))].
\]
Let $f^*_{\log} = \argmin_f L_{\log}(f)$ be the minimizer of $L_{\log}(f)$
among all functions $f \colon X \to \mathbb{R}$. We define the logistic \emph{regret}
of $f$ as:
\[
R_{\log}(f) = L_{\log}(f) - L_{\log}(f^*_{\log}).
\]

\section{Main Result}
\label{sec:main_result}

Any real-valued function $f \colon X \to \mathbb{R}$
can be turned into a classifier $h_{f,\theta} \colon X \to \{-1,1\}$, by thresholding at
some value $\theta$:
\[
 h_{f,\theta}(x) = \sgn(f(x) - \theta),
\]
where $\sgn(x)$ is the sign function, and we use the convention that $\sgn(0)=1$.

The purpose of this paper is to address the following problem: given a function $f$
with logistic regret $R_{\log}(f)$, and a threshold $\theta$, what is the maximum
AMS regret of $h_{f,\theta}$? In other words, can we bound $R_{\ams}(h_{f,\theta})$
in terms of $R_{\log}(f)$? 
We give a positive answer to this question, which based on the following regret bound:
\begin{lemma}
\label{lemma:main}
There exists a threshold $\theta^*$, such that for any $f$,
\[
 R_{\ams}(h_{{f,\theta^*}}) \leq \frac{s(h^*_{\ams})}{b(h^*_{\ams})} \sqrt{\frac{1}{2} R_{\log}(f)}.
\]
\end{lemma}
The proof is quite long and hence is postponed to Section \ref{sec:proof}.
Interestingly, the proof goes by an intermediate bound of the AMS regret by
a cost-sensitive classification regret, with misclassification costs proportional
to the gradient coordinates of the AMS.

Lemma \ref{lemma:main} has the following interpretation. 
If we are able to find a function $f$ with small logistic regret, 
we are guaranteed that there exists a threshold $\theta^*$ such that $h_{f,\theta^*}$ has small AMS regret. 
Note that the same threshold $\theta^*$ will work for any $f$, and the right hand side of the bound is \emph{independent} of $\theta^*$.
We are now ready to prove the main result of the paper:
\begin{theorem}
Given a real-valued function $f$, let $\hat{\theta} = \argmax_{\theta} \ams(h_{f,\theta})$. Then:
\[
R_{\ams}(h_{f,\hat{\theta}}) \leq \frac{s(h^*_{\ams})}{b(h^*_{\ams})} \sqrt{\frac{1}{2} R_{\log}(f)}.
\]
\label{thm:main} 
\end{theorem}
\begin{proof}
The result follows immediately from Lemma \ref{lemma:main} by noticing that
solving $\max_{\theta}\ams(h_{f,\theta})$ is equivalent to 
solving $\min_\theta R_{\ams}(h_{f,\theta})$, and that 
$\min_\theta R_{\ams}(h_{f,\theta}) \leq R_{\ams}(h_{f,\theta^*})$.
\end{proof}
Theorem \ref{thm:main} motivates the following procedure for AMS maximization:
\begin{enumerate}
\item Find $f$ with small logistic regret, e.g. by 
employing a learning algorithm minimizing logistic loss on the training sample.
\item Given $f$, solve $\hat{\theta} = \argmax_{\theta} \ams(h_{f,\theta})$.
\end{enumerate}
Theorem \ref{thm:main} states that the AMS regret of the classifier obtained by
this procedure
is upperbounded by the logistic regret of the underlying real-valued function. 

We now discuss how to approach step 2 of the procedure in practice.
In principle, this step requires maximizing AMS defined
by means of an unknown distribution $\Pr(x,y)$. However, it is sufficient to 
optimize $\theta$ on the empirical counterpart of AMS calculated on a separate validation sample.
Due to space limit, we only give a sketch of the proof of this fact: 
Step 2 involves optimization within a class of threshold functions
(since $f$ is fixed), which has VC-dimension equal to $2$ \citep{DevroyeGyorfiLugosi96}.
By convexity of $\ams^2$,
\begin{equation}
\label{eq:deviation_of_empirical_AMS_from_AMS}
\ams^2(s,b) - \ams^2(\hat{s},\hat{b}) \leq 
 \left(\frac{\partial \ams^2(s,b)}{\partial s}, \frac{\partial\ams^2(s,b)}{\partial b} \right)^\top 
 (s - \hat{s}, b - \hat{b})
\end{equation}
(see, e.g. \cite{BoydVandenberghe2004}), where $\hat{s}$ and $\hat{b}$ are empirical counterparts of $s$ and $b$. By VC theory,
the deviations of $\hat{s}$ from $s$, and $\hat{b}$ from $b$ can be upperbounded
 with high probability \emph{uniformly} over the class of all threshold functions by $O(1/\sqrt{m})$,
 where $m$ is the validation sample size. This and (\ref{eq:deviation_of_empirical_AMS_from_AMS})
 implies, 
 that $\ams^2(s,b)$ of the empirical maximizer
  is $O(1/\sqrt{m})$ close to the $\max_\theta \ams^2(h_{f,\theta})$. Hence, step 2 can be performed
 within $O(1/\sqrt{m})$ accuracy on a validation sample independent from the training sample.


\section{Proof of Lemma \ref{lemma:main}}
\label{sec:proof}

The proof consists of two steps. First, we bound the AMS regret of any classifier $h$ by
its cost-sensitive classification regret (introduced below). 
Next, we show that there exists a threshold $\theta^*$, such that for any $f$, 
the cost-sensitive classification regret of $h_{f,\theta^*}$ 
is upperbounded by the logistic regret of $f$.

\paragraph{Bounding AMS regret by cost-sensitive classification regret.}
Given a real number $c \in (0,1)$,
define a \emph{cost-sensitive classification loss}
$\ell_c \colon \{-1,1\} \times \{-1,1\} \to \mathbb{R}_+$ as:
\[
\ell_c(y,h) = c \bm[y=-1] \bm[h=1] + (1-c) \bm[y=1] \bm[h=-1],
\]
where $\bm[A]$ is the indicator function equal to $1$ if predicate $A$ is true, and $0$ otherwise.
The cost-sensitive loss assigns different costs of misclassification for positive and
negative labels. 
Given classifier $h$, the expected cost-sensitive loss of $h$ is:
\[
L_c(h) = \mathbb{E}_{(x,y)}[\ell_c(y,h(x))]
= c b(h) + (1-c) (\Pr(y=1) - s(h)),
\]
where $s(h)$ and $b(h)$ are true positive and false positive
rates defined before. Let $h^*_c = \argmin_h L_c(h)$ be the minimizer of
the expected cost-sensitive loss among all classifiers.
Define the cost-sensitive classification regret as:
\[
R_c(h) = L_c(h) - L_c(h^*_c).
\]
Any convex and differentiable function $g(x)$ satisfies
$g(x) \geq g(y) + \nabla g(y)^\top(x-y)$ for any $x,y$ in its convex domain \citep{BoydVandenberghe2004}.
Applying this inequality to $\ams^2(s,b)$ jointly convex in $(s,b)$, 
we have for any $s,b,s^*,b^* \in [0,1]$:
\begin{equation}
\label{eq:ams_cost-sensitive_bound}
 \ams^2(s,b) \geq \ams^2(s^*,b^*) +
 \left(\frac{\partial \ams^2(s^*,b^*)}{\partial s^*}, \frac{\partial\ams^2(s^*,b^*)}{\partial b^*} \right)^\top 
 (s - s^*, b - b^*).
\end{equation}
Given classifier $h$, we set $s = s(h), b = b(h), s^*=s(h^*_{\ams}), b^* = b(h^*_{\ams})$,
and:
\[
 C:= \frac{\partial \ams^2(s^*,b^*)}{\partial s^*} -\frac{\partial \ams^2(s^*,b^*)}{\partial b^*},
 \qquad
 c:= -\frac{1}{C} \frac{\partial \ams^2(s^*,b^*)}{\partial b^*}.
\]
Since $\ams^2(s,b)$ is increasing in $s$ and decreasing in $b$, both $\frac{\partial \ams^2(s^*,b^*)}{\partial s^*}$ and $-\frac{\partial \ams^2(s^*,b^*)}{\partial b^*}$
are positive, which implies $C > 0$ and $0 < c < 1$.
In this notation, (\ref{eq:ams_cost-sensitive_bound}) boils down to:
\begin{align*}
R_{\ams}(h) = \ams^2(h^*_{\ams}) - \ams^2(h) &\leq C \Big( c (b(h) - b(h^*_{\ams})) + (1-c) (s(h^*_{\ams})-s(h)) \Big)\\
&= C \big( L_c(h) - L_c(h^*_{\ams}) \big)\\
&\leq C \big( L_c(h) - L_c(h^*_c) \big) = C R_c(h),
\end{align*}
where the last inequality follows from the definition of $h^*_c$. Thus, the AMS regret
is upperbounded by the cost-sensitive classification regret with costs proportional to the gradient coordinates
of $\ams^2(s^*,b^*)$ at optimum $h^*_{\ams}$.\footnote{Note that the gradient at optimum
does not vanish, as the optimum is with respect to $h$, not $(s,b)$.}

\paragraph{Bounding cost-sensitive classification regret by logistic regret.}
We first give a bound on
cost-sensitive classification regret by means of logistic regret \emph{conditioned} at a given $x$.
This part relies on the techniques used by \cite{Bartlett_etal06}.
Then, the final bound is obtained by taking expectation with respect to $x$, and applying
Jensen's inequality.

Given a label $h \in \{-1,1\}$, and $\eta \in [0,1]$, define \emph{conditional} cost-sensitive classification
loss as:
\[
\ell_c(\eta,h) = c(1-\eta) \bm[h=1] + (1-c) \eta \bm[h=-1].
\]
The reason this quantity is called ``conditional loss'' becomes clear if we note that for any classifier
$h$, $L_c(h) = \mathbb{E}_x [\ell_c(\eta(x),h(x))]$, where $\eta(x) = \Pr(y=1|x)$. In other words,
$\ell_c(\eta(x),h(x))$ is the loss of $h$ conditioned on $x$.

Given $\eta$, let $h^*_c = \argmin_{h \in \{-1,1\}} \ell_c(\eta,h)$. 
 It can be easily verified
that:
\[
 h^*_c = \sgn\left(\eta - c \right),
\]
and $\ell_c(\eta,h^*_c) = \min\{c(1-\eta), (1-c) \eta \}$.
The conditional regret of $h$ is defined as $r_c(\eta,h) = \ell_c(\eta,h) - \ell_c(h^*_c)$.
Note that:
\[
r_c(\eta,h) = 
\left\{
\begin{array}{ll}
0 & \quad \text{if~~} h = h^*_c,\\
\left|\eta - c \right|
& \quad \text{if~~} h \neq h^*_c.
\end{array}
\right.
\]
Given a real number $f$, and $\eta \in [0,1]$, define \emph{conditional} logistic loss as:
\[
\ell_{\log}(\eta,f) = (1-\eta) \log\left(1 + e^f\right) + \eta \log\left(1 + e^{-f}\right).
\]
Let $f^*_{\log} = \argmin_{f \in \mathbb{R}} \ell_{\log}(\eta,f)$. By differentiating
$\ell_{\log}(\eta,f)$ with respect to $f$, and setting the derivative to $0$, we get that:
\[
 f^*_{\log} = \log \frac{\eta}{1-\eta},
\]
and $\ell_{\log}(\eta,f^*_{\log}) = -\eta \log \eta - (1-\eta) \log (1-\eta)$, the binary entropy of $\eta$.
The conditional logistic regret of $f$ is given by $r_{\log}(\eta,f) = \ell_{\log}(\eta,f) - \ell_{\log}(f^*_{\log})$.
The conditional regret has a particularly simple form 
when $f$ is re-expressed as a probability estimate $\eta_f$:
\[
r_{\log}(\eta,f) = D(\eta \| \eta_f),
\qquad \text{where}
\quad \eta_f := \frac{1}{1+e^{-f}},
\]
and $D(\eta \| \eta_f) = \eta \log \frac{\eta}{\eta_f} + (1-\eta) \log \frac{1-\eta}{1-\eta_f}$
is the Kullback-Leibler divergence. By Pinsker's inequality,
\[
D(\eta \| \eta_f) \geq 2 (\eta - \eta_f)^2. 
\]
Given real number $f$, define $h_{f,\theta^*} = \sgn(f-\theta^*)$,
where:
\[
 \theta^* = \log \frac{c}{1-c}.
\]
We will now
bound the conditional cost-sensitive classification regret
$r_c(\eta,h_{f,\theta^*})$ in terms of conditional logistic regret $r_{\log}(\eta,f)$.
First note that:
\[
h_{f,\theta^*} = 1 ~\iff~ f \geq \theta^* = \log \frac{c}{1-c}
~\iff~ \frac{1}{1+e^{-f}} \geq c
~\iff~ \eta_f \geq c,
\]
so that we can equivalently write $h_{f,\theta^*} = \sgn(\eta_f - c)$.
Since $h^*_c = \sgn(\eta - c)$,
then whenever
$(\eta_f - c)(\eta - c) > 0$, 
it holds $h_{f,\theta^*}= h^*_c$, and $r_c(\eta,h_{f,\theta^*}) = 0$.
On the other hand, when $(\eta_f - c)(\eta - c) \leq 0$,
it holds\footnote{
$r_c(\eta,h_{f,\theta^*}) = |\eta - c|$
if $(\eta_f - c)(\eta - c) < 0$, and can be either $0$ or
$|\eta - c|$ when $(\eta_f - c)(\eta - c) = 0$.
} 
$r_c(\eta,h_{f,\theta^*}) \leq |\eta - c|$, whereas:
\begin{align*}
r_{\log}(\eta,f) &= D(\eta \| \eta_f) \stackrel{\mathrm{Pinsker's}}{\geq} 2(\eta - \eta_f)^2 = 2(\eta - c + c - \eta_f)^2 \\
&= 2(\eta-c)^2 + 4(\eta - c)(c - \eta_{f})
+ 2(c - \eta_f)^2 \\
&\geq 2 (\eta - c)^2 
\geq 2 r^2_c(\eta,h_{f,\theta^*}),
\end{align*}
where the last but one inequality is implied by $(\eta_f - c)(\eta - c) \leq 0$. Taking both cases together, we get:
\[
r_c(\eta,h_{f,\theta^*}) \leq \sqrt{r_{\log}(\eta,f)/2}.
\]
Now, given any function $f$,
\begin{align*}
R_c(h_{f,\theta^*}) &= \mathbb{E}_x[r_c(\eta,h_{f,\theta^*})]
\leq \mathbb{E}_x\left[\sqrt{ r_{\log}(\eta,f)/2}\right] 
\leq \sqrt{\mathbb{E}_x[r_{\log}(\eta,f)]/2}
= \sqrt{R_{\log}(f)/2},
\end{align*}
where the last inequality is from Jensen's inequality applied to the concave
function $x \mapsto \sqrt{x}$.

\paragraph{Finishing the proof.}
Combining the results from both parts, we get:
\[
R_{\ams}(h_{f,\theta^*}) \leq C R_c(h_{f,\theta^*})
\leq C \sqrt{R_{\log}(f)/2},
\]
where $\theta^* = \log\frac{c}{1-c}$ is independent of $f$.
Recalling that
$C = \frac{\partial \ams^2(s^*,b^*)}{\partial s^*} -\frac{\partial \ams^2(s^*,b^*)}{\partial b^*}$, we calculate:
\[
C = \log\left(1 + \frac{s^*}{b^*}\right) - \left(\log\left( 1 + \frac{s^*}{b^*} \right) - \frac{s^*}{b^*} \right) = \frac{s^*}{b^*},
\]
where $s^* = s(h^*_{\ams})$ and $b^* = b(h^*_{\ams})$.
This finished the proof. 
$\square$

Note that the proof actually specifies the exact value of the universal threshold $\theta^*$:
\[
\theta^* = \log \frac{c}{1-c}, \qquad \text{where~} c = 1 - \frac{b^*}{s^*} \log \left(1+\frac{s^*}{b^*}\right).
\]

\section{Generalization beyond AMS and logistic loss}
\label{sec:generalization}
Results of this paper can be generalized beyond AMS metric and logistic loss surrogate. 
The AMS can be replaced by any other evaluation metric, which enjoys the following two properties:
1) is increasing in $s$, and decreasing in $b$; 2) is jointly convex in $s$ and $b$.
These were the only two properties of the AMS used in the proof of Lemma \ref{lemma:main}.
The logistic loss surrogate can be replaced by any other convex surrogate loss $\ell$, 
such that the following property holds: There exists a threshold $\theta^*$ which is a function of the cost $c$, such that for all $f$,
\[
R_c(h_{f,\theta^*}) \leq \lambda \sqrt{R_{\ell}(f)},
\]
for some positive constant $\lambda$. This property is satisfied by, e.g., squared error loss $\ell_{\mathrm{sq}}(y,f) = (y-f)^2$ with $\lambda = 1$,
which can be verified by noticing that the logistic regret upperbounds the squared error regret by Pinsker's inequality.
We conjecture that all \emph{strongly proper composite losses} \citep{Agarwal2014} hold this property.

\acks{The author was supported by the Foundation For Polish Science Homing Plus grant, co-financed by the European Regional Development Fund.
The author would like to thank Krzysztof Dembczy{\'n}ski for interesting discussions and proofreading the paper.}

\bibliography{higgs}

\begin{thebibliography}{11}
\providecommand{\natexlab}[1]{#1}
\providecommand{\url}[1]{\texttt{#1}}
\expandafter\ifx\csname urlstyle\endcsname\relax
  \providecommand{\doi}[1]{doi: #1}\else
  \providecommand{\doi}{doi: \begingroup \urlstyle{rm}\Url}\fi

\bibitem[Adam-Bourdarios et~al.(2014)Adam-Bourdarios, Cowan, Germain, Guyon,
  K{\'e}gl, and Rousseau]{AdamBourdarios_etal2014}
Claire Adam-Bourdarios, Glen Cowan, C{\'e}cile Germain, Isabelle Guyon,
  Bal{\'a}zs K{\'e}gl, and David Rousseau.
\newblock Learning to discover: the {H}iggs boson machine learning challenge,
  2014.
\newblock URL \url{http://higgsml.lal.in2p3.fr/documentation/}.

\bibitem[Agarwal(2014)]{Agarwal2014}
Shivani Agarwal.
\newblock Surrogate regret bounds for bipartite ranking via strongly proper
  losses.
\newblock \emph{Journal of Machine Learning Research}, 15:\penalty0 1653--1674,
  2014.
\newblock URL \url{http://jmlr.org/papers/v15/agarwal14b.html}.

\bibitem[Bartlett et~al.(2006)Bartlett, Jordan, and McAuliffe]{Bartlett_etal06}
Peter~L. Bartlett, Michael~I. Jordan, and Jon~D. McAuliffe.
\newblock Convexity, classification, and risk bounds.
\newblock \emph{Journal of the American Statistical Association}, 101\penalty0
  (473):\penalty0 138--156, 2006.

\bibitem[Boyd and Vandenberghe(2004)]{BoydVandenberghe2004}
Stephen Boyd and Lieven Vandenberghe.
\newblock \emph{Convex Optimization}.
\newblock Cambridge University Press, 2004.

\bibitem[Cowan et~al.(2011)Cowan, Cranmer, Gross, and Vitells]{Cowan_etal2011}
Glen Cowan, Kyle Cranmer, Eilam Gross, and Ofer Vitells.
\newblock Asymptotic formulae for likelihood-based tests of new physics.
\newblock \emph{The European Physical Journal C-Particles and Fields},
  71\penalty0 (2):\penalty0 1--19, 2011.

\bibitem[Devroye et~al.(1996)Devroye, Gy\"orfi, and
  Lugosi]{DevroyeGyorfiLugosi96}
Luc Devroye, Laszlo Gy\"orfi, and Gabor Lugosi.
\newblock \emph{A Probabilistic Theory of Pattern Recognition}.
\newblock Springer, 1st edition, 1996.

\bibitem[Hastie et~al.(2009)Hastie, Tibshirani, and
  Friedman]{FriedmanHastieTibshirani03}
Trevor Hastie, Robert Tibshirani, and Jerome~H. Friedman.
\newblock \emph{Elements of Statistical Learning: Data Mining, Inference, and
  Prediction}.
\newblock Springer, 2009.

\bibitem[Mackey and Bryan(2014)]{MackeyBryan2014}
Lester Mackey and Jordan Bryan.
\newblock Weighted classification cascades for optimizing discovery
  significance in the {H}iggs{ML} challenge.
\newblock \emph{CoRR}, abs/1409.2655, 2014.
\newblock URL \url{http://arxiv.org/abs/1409.2655}.

\bibitem[Narasimhan et~al.(2014)Narasimhan, Vaish, and
  Agarwal]{Narasimhan_etal2014}
Harikrishna Narasimhan, Rohit Vaish, and Shivani Agarwal.
\newblock On the statistical consistency of plug-in classifiers for
  non-decomposable performance measures.
\newblock In \emph{Neural Information Processing Systems (NIPS)}, 2014.

\bibitem[Natarajan et~al.(2014)Natarajan, Koyejo, Ravikumar, and
  Dhillon]{Natarajan_etal2014}
Nagarajan Natarajan, Oluwasanmi Koyejo, Pradeep~K. Ravikumar, and Inderjit~S.
  Dhillon.
\newblock Consistent binary classification with generalized performance
  metrics.
\newblock In \emph{Neural Information Processing Systems (NIPS)}, 2014.

\bibitem[Zhao et~al.(2013)Zhao, Edakunni, Pocock, and Brown]{Zhao_etal_2013}
Ming-Jie Zhao, Narayanan Edakunni, Adam Pocock, and Gavin Brown.
\newblock Beyond {F}ano's inequality: Bounds on the optimal {F}-score, {BER},
  and cost-sensitive risk and their implications.
\newblock \emph{Journal of Machine Learning Research}, 14:\penalty0 1033--1090,
  2013.

\end{thebibliography}

\end{document}